\documentclass{article}
\pdfoutput=1
\usepackage{amssymb,amsmath,amsthm,mathtools}
\usepackage[a4paper]{geometry}
\usepackage{tgtermes}
\usepackage{hyperref}
\newtheorem{lem}{Lemma}[section]
\newtheorem{thm}{Theorem}[section]
\newtheorem{defn}{Definition}[section]
\newtheorem{exam}{Example}[section]

\newcommand*{\dif}{\mathop{}\!\mathrm{d}}

\title{Convergence Rates of Training Deep Neural Networks via Alternating Minimization Methods}
\author{
Jintao Xu \thanks{Department of Mathematical Sciences, Tsinghua University, Beijing 100084, China. \newline Email: xujt19@mails.tsinghua.edu.cn}\and
Chenglong Bao \thanks{Yau Mathematical Sciences Center, Tsinghua University, Beijing 100084, China, and Yanqi Lake Beijing Institute of Mathematical Sciences and Applications, Beijing 101408, China. \newline Email:
clbao@mail.tsinghua.edu.cn}\and
Wenxun Xing \thanks{Department of Mathematical Sciences, Tsinghua University, Beijing 100084, China. \newline Email: wxing@mail.tsinghua.edu.cn}
}
\date{}

\begin{document}
\maketitle
\begin{abstract}
Training deep neural networks (DNNs) is an important and challenging optimization problem in machine learning due to its non-convexity and non-separable structure. The alternating minimization (AM) approaches split the composition structure of DNNs and have drawn great interest in the deep learning and optimization communities. In this paper, we propose a unified framework for analyzing the convergence rate of AM-type network training methods. Our analysis is based on the non-monotone $j$-step sufficient decrease conditions and the Kurdyka-\L ojasiewicz (KL) property, which relaxes the requirement of designing descent algorithms. We show the detailed local convergence rate if the KL exponent $\theta$ varies in $[0,1)$. Moreover, the local R-linear convergence is discussed under a stronger $j$-step sufficient decrease condition.
\end{abstract}

Non-monotone $j$-step sufficient decrease; Convergence rate\\


\section{Introduction}
\label{intro}
In recent years, deep learning has achieved impressive successes in many areas including computer vision \cite{Chen2018,He2016}, natural language processing \cite{Sutskever2014,Vaswani2017}, and recommender system \cite{Cheng2016,Covington2016}.
For deep neural networks (DNNs) training, the alternating minimization (AM)-type training methods, mainly based on the block coordinate descent (BCD) \cite{Shi2016} or the alternating direction method of multipliers (ADMM) \cite{Boyd2011}, have been discussed such as BCD-type algorithms \cite{CarreiraPerpinan2014,Gu2020,Lau2018,Wang2022,Zeng2019,Zhang2017} and ADMM-type algorithms \cite{Kiaee2016,Taylor2016,Wang2020,Wang2019,Zeng2021,Zhang2016}.
Carreira-Perpi\~{n}\'{a}n and Wang \cite{CarreiraPerpinan2014}, Lau et al. \cite{Lau2018}, Zeng et al. \cite{Zeng2019}, and Gu et al. \cite{Gu2020} designed BCD-type algorithms to train the feedforward neural networks (FNNs) approximately. Additionally, ADMM-type algorithms for FNNs are also proposed by Taylor et al. \cite{Taylor2016}, Wang et al. \cite{Wang2020,Wang2019}, and Zeng et al. \cite{Zeng2021}. Furthermore, AM-type training methods are designed to train other neural network models like the convolutional neural networks (CNNs) \cite{Gu2020,Kiaee2016} and the residual networks (ResNets) \cite{Zeng2019}. Besides, online AM-type training \cite{Choromanska2019} and parallel AM-type training \cite{Taylor2016,Wang2020} are also implemented.
In these methods, auxiliary variables are added for each layer to decouple the nested parameters in DNNs, and the vanishing gradient issue \cite{Bengio1994,Goodfellow2016} is avoided.

Let $\min_{X\in\mathcal{D}}f(X)$ be a DNNs training model, where $\mathcal{D}\subseteq\mathbb{R}^{m\times n}$, and $\{X_{k}\}$ be a sequence generated by an AM-type training algorithm. In this paper, we propose a unified framework for establishing the convergence rate of the objective function value $\{f(X_{k})\}$ generated by various AM-type training algorithms based on a non-monotone $j$-step sufficient decrease condition. Specifically, motivated by the algorithm mDLAM in \cite{Wang2022}, let $j$ be a positive integer, our analysis imposes a $j$-step sufficient decrease condition defined as follows, which relaxes the common descent condition used in \cite{Attouch2013}.
\begin{itemize}
\item[]\textbf{A1.} For a certain $j\in\mathbb{N}_{+}$, there exists $c_{1}>0$ such that
\begin{equation}\label{eqn:j-step}
c_{1}{\rm dist}(\textbf{0}, \partial f(X_{k+j}))^{2}\leq f(X_{k})-f(X_{k+j})
\end{equation}
for each $k\geq k_{0}$.
\end{itemize}
When $j=1$, the above condition can be derived from the sufficient decrease condition H1 and the relative error condition H2 of \cite{Attouch2013}, which is monotone. When $j\geq 2$, it allows the oscillations of $\{f(X_k)\}$ during the consecutive $j$ iterations, which can be classified as a non-monotone method. In Sect. \ref{sec:3}, we will show that the existing four AM-type algorithms for training DNNs satisfy A1. Besides, the Kurdyka-\L ojasiewicz (KL) property \cite{Attouch2010,Li2018} assumption of $f$, which implies a local sharpness under reparametrization \cite{Attouch2013}, plays a central role in our analysis.

Convergence results have been stated for some AM-type training algorithms \cite{Choromanska2019,Jagatap2018,Lau2018,Wang2020,Wang2022,Wang2019,Zeng2019,Zeng2021}. From the theoretical perspective, the convergence rate of $\{X_{k}\}$ generated by a BCD-type algorithm is analyzed \cite{Lau2018}, and the convergence of the objective function value $\{f(X_{k})\}$ is proven for BCD-type \cite{Zeng2019} and ADMM-type algorithms \cite{Wang2020,Zeng2021}. Sample complexity of an AM-type algorithm for ReLU networks is established in \cite{Jagatap2018}. Besides, the convergence rate of the expected error of an online AM-type training algorithm is analyzed in \cite{Choromanska2019}. However, there are few results about the convergence rate of the objective function value, like that in \cite{Wang2022}, and we address this question.

The main contribution of this paper is to propose a unified framework based on A1 for theoretically analyzing the convergence rate of the objective function value sequences generated by AM-type training methods. Thanks to the oscillations of the sequence being allowed, a wider range of training methods can be addressed theoretically and uniformly. We establish the local convergence rate
in Theorem \ref{thm:1}, which depends on the different values of the Kurdyka-\L ojasiewicz (KL) exponent $\theta\in[0,1)$ \cite{Attouch2010,Li2018} of $f$.
Moreover, if we replace the lower bound in \eqref{eqn:j-step} with $c_{2}{\rm dist}(\textbf{0}, \partial f(X_{k+j}))^{\frac{1}{\alpha}}$, where $c_{2}$ and $\alpha\geq\theta$ are positive constants, we give the local linear convergence in Theorem \ref{thm:2}.

Additionally, the above theoretical results are true for those AM-type training methods with a non-monotone $j$-step ($j\ge 2$) sufficient decrease condition. In this way, both monotone and non-monotone training algorithms can be handled uniformly.

The rest of the paper is organized as follows. Notations and definitions used throughout this paper are listed in Sect. \ref{sec:2}. Four examples are shown in Sect. \ref{sec:3}. Estimations of convergence rate are stated in Sect. \ref{sec:4}, and we summarize our results in Sect. \ref{sec:5}.

\section{Notations and Definitions}
\label{sec:2}
We give notations and definitions that are useful in our analysis.\\
\textbf{Notations.} Throughout this paper, $\mathbb{R}$, $\mathbb{R}^{m}$, $\mathbb{R}^{m\times n}$, and $\mathbb{N}_{+}$ denote the set of real numbers, real $m$-dimensional vectors, real $m\times n$ matrices, and positive integers respectively. \textbf{0} denotes the matrix of all zeros whose size varies from the context. $\Vert\cdot\Vert$ denotes the Euclidean norm for $x\in\mathbb{R}^{m}$ and the Frobenius norm for $X\in\mathbb{R}^{m\times n}$, respectively. For $A, B\in\mathbb{R}^{m\times n}, \langle A, B\rangle={\rm tr }(AB^{\mathrm{T}})$. ${\rm dist}(x, \mathcal{S})\coloneqq\inf_{s\in\mathcal{S}}\Vert x-s\Vert$. For any $u\in\mathbb{R}$, $\lfloor u\rfloor$ is the greatest integer no larger than $u$. $\mathcal{O}(\cdot)$ is the standard big O asymptotic notation.
\begin{defn}{\textbf{\rm(Fr{\'{e}}chet subdifferential \cite{Mordukhovich2006,Rockafellar1998})}}\label{defn:1}\\
\rm The Fr{\'{e}}chet subdifferential of $f$ at $X\in {\rm dom}(f)$ is the set
\begin{equation*}
\widehat{\partial}f(X)=\left\{G\left|\liminf_{Y\neq X\atop Y\to X}\frac{f(Y)-f(X)-\langle G, Y-X\rangle}{\Vert Y -X\Vert}\geq0\right.\right\}.
\end{equation*}
\end{defn}
\begin{defn}{\textbf{\rm (Limiting subdifferential \cite{Mordukhovich2006,Rockafellar1998})}}\label{defn:2}\\
\rm For each $X\in {\rm dom}(f)$, the limiting subdifferential of $f$ at $X$ is the set
\begin{equation*}
\partial f(X)=\left\{G\left|\exists X_{k}\to X, f(X_{k})\to f(X), G_{k}\to G, G_{k}\in\widehat{\partial}f(X_{k})\right.\right\}.
\end{equation*}
${\rm dom}(\partial f)\coloneqq\{X|\partial f(X)\neq\emptyset\}$.
\end{defn}
\begin{defn}{\textbf{\rm (Limiting critical point \cite{Attouch2013})}}\label{defn:3}\\
\rm A point $X\in{\rm dom}(f)$ is called a (limiting) critical point of $f$ if $\textbf{0}\in\partial f(X)$.
\end{defn}
\begin{defn}{\textbf{\rm(Kurdyka-\L ojasiewicz property \cite{Attouch2010,Li2018})}}\label{defn:4}\\
\rm A proper lower semicontinuous function $f$ is said to have the Kurdyka-\L ojasiewicz (KL) property with exponent $\theta$ at $\widehat{X}\in {\rm dom}(\partial f)$ if there exist $c, \tau\in(0, \infty]$, $\theta\in[0, 1)$, and a neighborhood $\mathcal{U}$ of $\widehat{X}$ such that for all $X\in \mathcal{U}\cap\{X|f(\widehat{X})<f(X)<f(\widehat{X})+\tau\}$,
\begin{equation*}
(f(X)-f(\widehat{X}))^{\theta}\leq c{\rm dist}(\textbf{0}, \partial f(X)).
\end{equation*}
We call $\theta$ as the Kurdyka-\L ojasiewicz (KL) exponent at $\widehat{X}$ \cite{Li2018}.

KL property is widely used in non-convex optimization \cite{Attouch2010,Attouch2013,Bolte2014,Xu2013}. The pioneering work on it is credited to Kurdyka \cite{Kurdyka1998} and \L ojasiewicz \cite{Stanislaw1963,Stanislaw1993}. A large class of functions is proven to satisfy the KL property, for example, real analytic functions \cite{Stanislaw1963,Stanislaw1984,Stanislaw1993}, functions definable in o-minimal structures \cite{Bolte2007a,Kurdyka1998}, uniformly convex functions \cite{Bolte2014}, and subanalytic continuous functions \cite{Bolte2007}.
In the scenario of DNNs training, the linear, polynomial, hyperbolic tangent, and sigmoid activation functions; the squared, logistic, and exponential loss functions; and the squared Frobenius norm regularization terms all satisfy the KL property \cite{Zeng2019,Zeng2021}. More details about the KL property can be seen in \cite{Attouch2010,Attouch2013,Bolte2014} and the references therein.
\end{defn}
\begin{defn}{\textbf{\rm (Local convergence)}}\label{defn:5}\\
\rm For a convergent sequence $\{X_k\}$ generated by an algorithm $\mathcal{A}$ with a limit $X^{*}$, if the initial point $X_0$ is needed to be close enough to $X^*$, algorithm $\mathcal{A}$ is said to be local convergence ( $\{X_k\}$ locally convergent to $X^*$).
\end{defn}
\begin{defn}{\textbf{\rm (Root (R)-convergence rate \cite{Sun2006})}}\label{defn:6}\\
\rm For any convergent sequence $\{X_{k}\}$ with a limit $X^{*}$, $R_{1}\coloneqq\limsup_{k\to\infty}\Vert X_{k}-X^{*}\Vert^{\frac{1}{k}}$. If $0<R_{1}<1$, $\{X_{k}\}$ is called Root (R)-linearly convergent. If $R_{1}=1$, $\{X_{k}\}$ is called Root (R)-sublinearly convergent.
\end{defn}

\section{Typical AM-type algorithms}
\label{sec:3}
In this section, we present four examples that apply AM-type algorithms for training DNNs. These examples and their numerical results motivate us to construct a unified framework for estimating the convergence rate of the objective function value sequences generated by AM-type DNNs training algorithms.

\begin{exam}{\textbf{\rm(BCD for FNNs \cite{Zeng2019})}}\label{exam:1}\\
\rm For the feedforward neural networks (FNNs) training, Zeng et al. \cite{Zeng2019} formulated two optimization models named two-splitting and three-splitting formulations, and designed BCD-type algorithms for their unconstrained approximations, respectively. For the two-splitting formulation, the objective function is
\begin{align}\label{obj:1}
f(X)=&\frac{1}{n}\sum_{j=1}^{n}\ell((V_{N})_{:j}, y_{j})+\sum_{i=1}^{N}r_{i}(W_{i})+\sum_{i=1}^{N}s_{i}(V_{i})\nonumber\\
&+\frac{\gamma}{2}\sum_{i=1}^{N}\Vert V_{i}-\sigma_{i}(W_{i}V_{i-1})\Vert^{2},
\end{align}
where $X=(\{W_{i}\}_{i=1}^{N}, \{V_{i}\}_{i=1}^{N})$, $\sigma_{i}$ denotes the activation function of the $i$th layer, $i=1, 2, \ldots, N$, $\ell$ is a loss function, $r_{i}$, $s_{i}$ can be seen as regularization terms about $W_{i}$, $V_{i}$, $i=1, 2, \ldots, N$, respectively, $(V_{N})_{:j}$ denotes the $j$th column of $V_{N}$, $j=1, 2, \ldots, n$, and the last term represents a quadratic penalty for constraints $V_{i}=\sigma_{i}(W_{i}V_{i-1})$, $i=1, 2, \ldots, N$.
Under the assumptions in the Theorem 1 in \cite{Zeng2019}, (\ref{obj:1}) satisfies the KL property on any closed set. Moreover, $\{f(X_{k})\}$ generated by the BCD-type algorithm is convergent, and satisfies
\begin{equation*}
\frac{a}{b^{2}}{\rm dist}(\textbf{0}, \partial f(X_{k+1}))^{2}\leq f(X_{k})-f(X_{k+1})
\end{equation*}
for certain $a, b>0$ (see \cite{Zeng2019} for the values of $a$ and $b$), which is an A1 with $j=1$.

And for the three-splitting formulation,
\begin{align}\label{obj:2}
f(X)=&\frac{1}{n}\sum_{j=1}^{n}\ell((V_{N})_{:j}, y_{j})+\sum_{i=1}^{N}r_{i}(W_{i})+\sum_{i=1}^{N}s_{i}(V_{i})\nonumber\\
&+\frac{\gamma}{2}\sum_{i=1}^{N}\Vert V_{i}-\sigma_{i}(U_{i})\Vert^{2}+\frac{\gamma}{2}\sum_{i=1}^{N}\Vert U_{i}-W_{i}V_{i-1}\Vert^{2},
\end{align}
where $X=(\{W_{i}\}_{i=1}^{N}, \{U_{i}\}_{i=1}^{N}, \{V_{i}\}_{i=1}^{N})$, and the last two terms are quadratic penalties for constraints $V_{i}=\sigma_{i}(U_{i})$ and $U_{i}=W_{i}V_{i-1}$, $i=1, 2, \ldots, N$, respectively. Under the same assumptions, (\ref{obj:2}) satisfies the KL property on any closed set. Similarly, $\{f(X_{k})\}$ is convergent, and satisfies
\begin{equation*}
\frac{a}{b^{2}}{\rm dist}(\textbf{0}, \partial f(X_{k+1}))^{2}\leq f(X_{k})-f(X_{k+1})
\end{equation*}
for certain $a, b>0$ (see \cite{Zeng2019} for the values of $a$ and $b$), which is an A1 with $j=1$.
\end{exam}
\begin{exam}{\textbf{\rm (ADMM for FNNs \cite{Zeng2021})}}\label{exam:2}\\
\rm Zeng et al. \cite{Zeng2021} considered the augmented Lagrangian function of an FNNs training model and solved it via an ADMM-type algorithm. Technically, they gave
\begin{align}\label{obj:3}
f(X)=&\frac{1}{2}\Vert V_{N}-Y\Vert^{2}+\frac{\lambda}{2}\sum_{i=1}^{N}\Vert W_{i}\Vert^{2}+\sum_{i=1}^{N-1}\langle  \Lambda_{i}, \sigma(W_{i}V_{i-1})-V_{i}\rangle\nonumber\\
&+\sum_{i=1}^{N-1}\frac{\beta_{i}}{2}\Vert\sigma(W_{i}V_{i-1})-V_{i}\Vert^{2}+\langle\Lambda_{N}, W_{N}V_{N-1}-V_{N}\rangle\nonumber\\
&+\frac{\beta_{N}}{2}\Vert W_{N}V_{N-1}-V_{N}\Vert^{2}+\sum_{i=1}^{N}\xi_{i}\Vert V_{i}-\overline{V}_{i}\Vert^{2},
\end{align}
where $X=(\{W_{i}\}_{i=1}^{N}, \{V_{i}\}_{i=1}^{N}, \{\Lambda_{i}\}_{i=1}^{N}, \{\overline{V}_{i}\}_{i=1}^{N})$ and $\sigma$ denotes the activation function. Suppose that there exist $\chi>0$ and $k_{0}\in\mathbb{N}$ such that $\Vert V_{i}^{k-1}-V_{i}^{k-2}\Vert\leq\chi\Vert V_{i}^{k}-V_{i}^{k-1}\Vert$ for each $k\geq k_{0}$. Under the assumptions in the Theorem 7 in \cite{Zeng2021}, (\ref{obj:3}) satisfies the KL property, and $\{f(X_{k})\}$ generated by the ADMM-type algorithm is convergent. Moreover, there exist $a, b>0$ (see \cite{Zeng2021} for the values of $a$ and $b$) such that
\begin{equation*}
\frac{a}{2b^{2}(1+\chi)^{2}N}{\rm dist}(\textbf{0}, \partial f(X_{k+1}))^{2}\leq f(X_{k})-f(X_{k+1})
\end{equation*}
for each $k\geq k_{0}$, which is an A1 with $j=1$.
\end{exam}
\begin{exam}{\textbf{\rm(mDLAM for FNNs \cite{Wang2022})}}\label{exam:3}\\
\rm Wang et al. \cite{Wang2022} formulated an FNNs training model and designed an AM-type algorithm called mDLAM to solve it. The objective function of the training model is
\begin{equation}\label{obj:4}
f(X)=\ell(v_{N}, y)+\sum_{i=1}^{N}r_{i}(W_{i})+\frac{\gamma}{2}\sum_{i=1}^{N}\Vert v_{i}-W_{i}u_{i-1}\Vert^{2},
\end{equation}
where $X=(\{W_{i}\}_{i=1}^{N}, \{u_{i}\}_{i=1}^{N-1}, \{v_{i}\}_{i=1}^{N})$. The last term is a quadratic penalty for constraints $v_{i}=W_{i}u_{i-1}$, $i=1, 2, \ldots, N$, and the operations of non-linear continuous activation functions $\sigma_i$ are formulated as inequality constraints $\sigma_i(v_i)-\epsilon\leq u_i\leq \sigma_i(v_i)+\epsilon, i=1, 2, \ldots, N-1$. If (\ref{obj:4}) is real analytic \cite{Krantz2002}, it satisfies the KL property. Moreover, according to the Lemma 2 and inequality (26) in \cite{Wang2022}, $\{f(X_{k})\}$ is convergent, and satisfies
\begin{equation*}
\frac{a}{6b^{2}}{\rm dist}(\textbf{0}, \partial f(X_{k+2}))^{2}\leq f(X_{k})-f(X_{k+2})
\end{equation*}
for certain $a, b>0$ (see \cite{Wang2022} for the values of $a$ and $b$), which is a non-monotone 2-step sufficient decrease condition.
\end{exam}
\begin{exam}{\textbf{\rm(BCD for ResNets \cite{Zeng2019})}}\label{exam:4}\\
\rm For the residual networks (ResNets) \cite{He2016} training, Zeng et al. \cite{Zeng2019} formulated a three-splitting simplified model and its unconstrained approximation. The objective function is
\begin{align}\label{obj:5}
f(X)=&\frac{1}{n}\sum_{j=1}^{n}\ell((V_{N})_{:j}, y_{j})+\sum_{i=1}^{N}r_{i}(W_{i})+\sum_{i=1}^{N}s_{i}(V_{i})\nonumber\\
&+\frac{\gamma}{2}\sum_{i=1}^{N}\Vert V_{i}-V_{i-1}-\sigma_{i}(U_{i})\Vert^{2}+\frac{\gamma}{2}\sum_{i=1}^{N}\Vert U_{i}-W_{i}V_{i-1}\Vert^{2},
\end{align}
where $X=(\{W_{i}\}_{i=1}^{N}, \{U_{i}\}_{i=1}^{N}, \{V_{i}\}_{i=1}^{N})$, and the last two terms represent quadratic penalties for constraints $V_{i}-V_{i-1}=\sigma_{i}(U_{i})$ and $U_{i}=W_{i}V_{i-1}$, $i=1, 2, \ldots, N$, respectively.
Under the assumptions in the Theorem 1 in \cite{Zeng2019}, (\ref{obj:5}) satisfies the KL property on any closed set. Moreover, $\{f(X_{k})\}$ generated by the BCD-type algorithm is convergent, and satisfies
\begin{equation*}
\frac{a}{3Nb^{2}}{\rm dist}(\textbf{0}, \partial f(X_{k+1}))^{2}\leq f(X_{k})-f(X_{k+1})
\end{equation*}
for certain $a, b>0$ (see \cite{Zeng2019} for the values of $a$ and $b$), which is an A1 with $j=1$.
\end{exam}
\section{Theoretical analysis}
\label{sec:4}
In this section, we give the following unified convergence rate estimation framework based on A1 and the KL property for  AM-type training algorithms.
\begin{thm}\label{thm:1}
For a proper lower semicontinuous objective function $f$ and a sequence $\{X_{k}\}$ generated by an AM-type training algorithm, suppose that there exists $\widehat{X}\in {\rm dom}(\partial f)$ such that $f$ satisfies the KL property at $\widehat{X}$ with a neighborhood $\mathcal{U}_{\widehat{X}}$ and the KL exponent $\theta$, $f(X_{k})\to f(\widehat{X})$ as $k\to\infty$, $X_{k}\in \mathcal{U}_{\widehat{X}}$ for each $k\geq k_{0}$, and  A1 is satisfied. Then the following conclusions hold:
\begin{itemize}
\item[(\romannumeral1)] If $\theta=0$, then $\{f\left(X_{k}\right)\}$ converges in a finite number of steps;
\item[(\romannumeral2)] If $\theta\in(0, \frac{1}{2}]$, then there exist $k_{1}\in\mathbb{N}$, $\eta\in(0, 1)$, and $C>0$ such that $f(X_{k})-f(\widehat{X})\leq C\eta^{\lfloor\frac{k-k_{1}}{j}\rfloor+1}$ for each $k\geq k_{1}$;
\item[(\romannumeral3)] If $\theta\in(\frac{1}{2}, 1)$, then there exist $k_{1}\in\mathbb{N}$ and $C>0$ such that $f(X_{k})-f(\widehat{X})\leq C(\lfloor\frac{k-k_{1}}{j}\rfloor+1)^{-\frac{1}{2\theta-1}}$ for each $k\geq k_{1}$.
\end{itemize}
Additionally, for each accumulation point (if any) $\widetilde{X}\in\mathcal{U}_{\widehat{X}}$ of $\{X_{k}\}$, it is a critical point of $f$ if and only if $f(\widetilde{X})=f(\widehat{X})$.
\end{thm}
With the similar arguments as in the proof of Theorem 3 of \cite{Li2016}, Theorem 2 of \cite{Attouch2009} for convergence rate estimation and Theorem 1 of \cite{Wang2022} for the analysis of accumulation points, we give the following proofs. Lemmas \ref{lem:1} and \ref{lem:2} below are needed.
\begin{lem}\label{lem:1}
Under the assumptions in Theorem \ref{thm:1}, we have $f(X_k)\geq f(\widehat{X}), k\geq k_0$. If $\{f(X_{k})\}$ be an infinite sequence, there exists $k_1\in\mathbb{N}$ such that for each $f(X_{k})-f(\widehat{X})>0, k\geq k_1$,
\begin{align}
(f(X_{k})-f(\widehat{X}))^{2\theta}\leq \frac{c^{2}}{c_{1}}(f(X_{k-j})-f(X_{k})).\label{eqn:key-inequality}
\end{align}
\end{lem}
\begin{proof}
First of all, the lower-boundedness of $\{f(X_{k})\}$ is proven.
By A1, $f(X_{k})\geq f(X_{k+j})\geq f(X_{k+2j})\geq\ldots\geq f(X_{k+nj})$ holds for each $k\geq k_{0}$. Letting $n\to\infty$, we have $f(X_{k})\geq f(\widehat{X}), k\geq k_{0}$.

If $\{f(X_{k})\}$ be an infinite sequence, there exists a $k_{1}\in\mathbb{N}, k_{1}\geq k_{0}+j$ such that $f(X_{k})-f(\widehat{X})<1$, and for each $f(X_{k})-f(\widehat{X})>0, k\geq k_{1}$,
\begin{align*}
(f(X_{k})-f(\widehat{X}))^{2\theta}&\leq c^{2}{\rm dist}(\textbf{0}, \partial f(X_{k}))^{2}\nonumber\\
&\leq \frac{c^{2}}{c_{1}}(f(X_{k-j})-f(X_{k})),
\end{align*}
where the first inequality follows from the KL property of $f$ at $\widehat{X}$, and the second inequality follows from A1.
\end{proof}
\begin{lem}[\cite{Attouch2013}]\label{lem:2}
Suppose that $X_{k}\to X$, $G_{k}\to G$, and $f(X_{k})\to f(X)$ as $k\to\infty$, of which $G_{k}\in\partial f(X_{k})$ for each $k$. Then $G\in\partial f(X)$.
\end{lem}
\begin{proof}{(proof of Theorem \ref{thm:1})}
(\romannumeral1) is proven by contradiction. If the conclusion is not true, there exists a subsequence $\{k_{l}\}\subseteq\{k_{1}, k_{1}+1, \ldots\}$ such that
\begin{equation*}
1\leq c{\rm dist}(\textbf{0}, \partial f(X_{k_{l}})).
\end{equation*}
Letting $l\to\infty$, $1\leq 0$, a contradiction. Therefore, there exists $k_{2}\in\mathbb{N}$ such that $f(X_{k})\equiv f(\widehat{X})$ for each $k\geq k_{2}$.

If $\{f(X_{k})\}$ be a finite sequence, (\romannumeral2) and (\romannumeral3) hold trivially. Then we only need to prove them in the case of infinite convergence.

When $\theta\in(0, \frac{1}{2}]$, according to (\ref{eqn:key-inequality}) in Lemma \ref{lem:1},
\begin{equation*}
f(X_{k})-f(\widehat{X})\leq(f(X_{k})-f(\widehat{X}))^{2\theta}\leq \frac{c^{2}}{c_{1}}(f(X_{k-j})-f(X_{k}))
\end{equation*}
holds for each $k\geq k_{1}$, which then implies that
\begin{equation*}
f(X_{k})-f(\widehat{X})\leq\frac{c^{2}}{c_{1}+c^{2}}(f(X_{k-j})-f(\widehat{X})), k\geq k_{1}.
\end{equation*}
Hence, we have
\begin{equation*}
f(X_{k})-f(\widehat{X})\leq C_{1} \left(\frac{c^{2}}{c_{1}+c^{2}}\right)^{\left\lfloor\frac{k-k_{1}}{j}\right\rfloor+1}
\end{equation*}
for each $k\geq k_{1}$, where
\begin{equation*}
C_{1}=\max\left\{f(X_{k_{1}-j})-f(\widehat{X}), f(X_{k_{1}-j+1})-f(\widehat{X}), \ldots, f(X_{k_{1}-1})-f(\widehat{X})\right\}.
\end{equation*}
It follows from A1 and the infinite convergence assumption of $\{f(X_{k})\}$ that $C_{1}>0$. Thus (\romannumeral2) holds with $C=C_{1}, \eta=\frac{c^{2}}{c_{1}+c^{2}}$.

When $\theta\in(\frac{1}{2}, 1)$, given a constant $\omega\in[2, \infty)$, for each $f(X_{k})-f(\widehat{X})>0, k\geq k_{1}$, if $(f(X_{k})-f(\widehat{X}))^{-2\theta}\leq\omega(f(X_{k-j})-f(\widehat{X}))^{-2\theta}$, then,
\begin{align*}
\frac{c_{1}}{c^{2}}&\leq(f(X_{k})-f(\widehat{X}))^{-2\theta}\left((f(X_{k-j})-f(\widehat{X}))-(f(X_{k})-f(\widehat{X}))\right)\\
&\leq\omega(f(X_{k-j})-f(\widehat{X}))^{-2\theta}\left((f(X_{k-j})-f(\widehat{X}))-(f(X_{k})-f(\widehat{X}))\right)\\
&\leq\omega\int_{f(X_{k})-f(\widehat{X})}^{f(X_{k-j})-f(\widehat{X})}x^{-2\theta}\dif x\\
&=\frac{\omega}{2\theta-1}\left((f(X_{k})-f(\widehat{X}))^{-2\theta+1}-(f(X_{k-j})-f(\widehat{X}))^{-2\theta+1}\right),
\end{align*}
where the first inequality follows from (\ref{eqn:key-inequality}) in Lemma \ref{lem:1}.
Hence
\begin{equation}
0<\frac{c_{1}(2\theta-1)}{c^{2}\omega}\leq(f(X_{k})-f(\widehat{X}))^{-2\theta+1}-(f(X_{k-j})-f(\widehat{X}))^{-2\theta+1}.\label{eqn:1}
\end{equation}
If $(f(X_{k})-f(\widehat{X}))^{-2\theta}>\omega(f(X_{k-j})-f(\widehat{X}))^{-2\theta}$, then,
\begin{equation*}
(f(X_{k})-f(\widehat{X}))^{-2\theta+1}\geq\omega^{\frac{2\theta-1}{2\theta}}(f(X_{k-j})-f(\widehat{X}))^{-2\theta+1}.
\end{equation*}
Hence
\begin{align}
0&<(\omega^{\frac{2\theta-1}{2\theta}}-1)C_{1}^{-2\theta+1}\nonumber\\
&\leq(\omega^{\frac{2\theta-1}{2\theta}}-1)(f(X_{k-j})-f(\widehat{X}))^{-2\theta+1}\nonumber\\
&\leq(f(X_{k})-f(\widehat{X}))^{-2\theta+1}-(f(X_{k-j})-f(\widehat{X}))^{-2\theta+1}.\label{eqn:2}
\end{align}
According to (\ref{eqn:1}) and (\ref{eqn:2}),
\begin{equation*}
(f(X_{k-j})-f(\widehat{X}))^{-2\theta+1}+L\leq(f(X_{k})-f(\widehat{X}))^{-2\theta+1}
\end{equation*}
holds for each $f(X_{k})-f(\widehat{X})>0, k\geq k_{1}$, where
\begin{equation*}
L=\min\left\{\frac{c_{1}(2\theta-1)}{c^{2}\omega}, (\omega^{\frac{2\theta-1}{2\theta}}-1)C_{1}^{-2\theta+1}\right\}>0.
\end{equation*}
Then we have
\begin{align}
f(X_{k})-f(\widehat{X})&\leq\left(C_{1}^{-2\theta+1}+L\left(\left\lfloor\frac{k-k_{1}}{j}\right\rfloor+1\right)\right)^{-\frac{1}{2\theta-1}}\nonumber\\
&\leq L^{-\frac{1}{2\theta-1}}\left(\left\lfloor\frac{k-k_{1}}{j}\right\rfloor+1\right)^{-\frac{1}{2\theta-1}}.\label{eqn:3}
\end{align}
Clearly, for each $f(X_{k})-f(\widehat{X})=0, k\geq k_{1}$, (\ref{eqn:3}) still holds. Thus we obtain (\romannumeral3) with $C=L^{-\frac{1}{2\theta-1}}$.

For each accumulation point $\widetilde{X}$, there exists a subsequence $\{X_{k_{l}}\}$ such that $\lim_{l\to\infty}X_{k_{l}}=\widetilde{X}$. By A1, $\lim_{l\to\infty}{\rm dist}(\textbf{0}, \partial f(X_{k_{l}}))=0$. For each $k_{l}$, there exists $G_{k_{l}}\in\partial f(X_{k_{l}})$ such that
\begin{equation*}
{\rm dist}(\textbf{0}, \partial f(X_{k_{l}}))\leq \Vert G_{k_{l}}\Vert\leq {\rm dist}(\textbf{0}, \partial f(X_{k_{l}}))+\frac{1}{k_{l}}.
\end{equation*}
Letting $l\to\infty$, we have
\begin{equation*}
0=\lim_{l\to\infty}{\rm dist}(\textbf{0}, \partial f(X_{k_{l}}))\leq \lim_{l\to\infty}\Vert G_{k_{l}}\Vert\leq \lim_{l\to\infty}\left({\rm dist}(\textbf{0}, \partial f(X_{k_{l}}))+\frac{1}{k_{l}}\right)=0.
\end{equation*}
So $\lim_{l\to\infty}\Vert G_{k_{l}}\Vert=0$. Without loss of generality, suppose that $G_{k_{l}}\to\widetilde{G}$ as $l\to\infty$. Then $\Vert\widetilde{G}\Vert=\lim_{l\to\infty}\Vert G_{k_{l}}\Vert=0$, $\widetilde{G}=\textbf{0}$. According to
\begin{equation*}
X_{k_{l}}\to\widetilde{X}, G_{k_{l}}\to\textbf{0}, \text{and}~ f(X_{k_{l}})\to f(\widetilde{X}), {\rm as}~ l\to\infty, G_{k_{l}}\in\partial f(X_{k_{l}}),
\end{equation*}
and Lemma \ref{lem:2}, we have $\textbf{0}\in\partial f(\widetilde{X})$.

Moreover, if $\widetilde{X}$ is a critical point, it follows from the Remark 2.5 (d) of \cite{Attouch2013} that $f(\widetilde{X})=f(\widehat{X})$.
\end{proof}
Parts (\romannumeral2) and (\romannumeral3) of the above theorem implies that $\{f(X_{k})\}$ converges at least locally R-linearly and locally R-sublinearly to $f(\widehat{X})$, respectively.
For each example in Sect. \ref{sec:3}, A1 and the KL property are satisfied, and the KL exponent $\theta$ of real analytic function is in $[\frac{1}{2}, 1)$ at a critical point \cite{Attouch2010,Stanislaw1963}.
Furthermore, $\{X_{k}\}$ is convergent in Examples 1, 2, and 4 \cite{Zeng2019,Zeng2021}, so the corresponding $\{f(X_k)\}$ has $\mathcal{O}(\eta^{k})$ local convergence rate for $\theta=\frac{1}{2}$ and $\mathcal{O}(k^{-\frac{1}{2\theta-1}})$ local convergence rate for $\theta\in(\frac{1}{2}, 1)$ by our Theorem \ref{thm:1}. Besides, if the assumption $X_{k}\in \mathcal{U}_{\widehat{X}}$ for each sufficiently large $k$ is satisfied in Example 3, the aforementioned results also hold, and the Theorem 2 in \cite{Wang2022} is a special case of our Theorem \ref{thm:1}. Moreover, the continuity of $f$ is satisfied in each example in Sect. \ref{sec:3} under certain assumptions \cite{Wang2022,Zeng2019,Zeng2021}. Then, each accumulation point of $\{X_{k}\}$ is a critical point by our Theorem \ref{thm:1} (see \cite{Wang2022,Zeng2019,Zeng2021} for the existence of accumulation points and their properties for each example in Sect. \ref{sec:3}).

Moreover, if the following stronger non-monotone $j$-step sufficient decrease condition is satisfied, $\{f(X_k)\}$ converges at least locally R-linearly to $f(\widehat{X})$ for each $\theta\in[0, 1)$ as shown in the following Theorem \ref{thm:2}.
\begin{itemize}
\item[]\textbf{A2.} For a certain $j\in\mathbb{N}_{+}$, there exist positive constants $\alpha\in[\theta, \infty) ~\text{and}~ c_{2}$ such that
\begin{equation*}
c_{2}{\rm dist}(\textbf{0}, \partial f(X_{k+j}))^{\frac{1}{\alpha}}\leq f(X_{k})-f(X_{k+j})
\end{equation*}
for each $k\geq k_{0}$, where $\theta$ is the KL exponent of $f$.
\end{itemize}
When $\theta\in(\frac{1}{2}, 1)$, compared with A1, a larger descent of $j$ steps iteration is guaranteed in A2, and it is a more dedicated estimation for $f(X_{k})-f(X_{k+j})$.
\begin{thm}\label{thm:2}
For an objective function $f$ and a sequence $\{X_{k}\}$, suppose that $f$ satisfies the KL property at $\widehat{X}$ with $\mathcal{U}_{\widehat{X}}$ and $\theta$, $f(X_{k})\to f(\widehat{X})$, $X_{k}\in\mathcal{U}_{\widehat{X}}$ for each $k\geq k_{0}$, and A2 is satisfied. Then  $\{f(X_{k})\}$ has a local convergence rate of $\mathcal{O}(\eta^{k})$, where $\eta\in(0, 1)$. Additionally, for each accumulation point (if any) $\widetilde{X}\in\mathcal{U}_{\widehat{X}}$ of $\{X_{k}\}$, it is a critical point if and only if $f(\widetilde{X})=f(\widehat{X})$.
\end{thm}
\begin{proof}
When $\theta=0$, with the similar arguments as in the proof of Theorem \ref{thm:1}, finite convergence is achieved, and $\mathcal{O}(\eta^{k})$ complexity bound holds trivially. When $\theta\in(0, 1)$, similarly, we only prove it in the case of infinite convergence. By A2 and the KL property of $f$ at $\widehat{X}$, there exists a $k_{1}\in\mathbb{N}, k_{1}\geq k_{0}+j$ such that for each $k\geq k_{1}$, ${\rm dist}(\textbf{0}, \partial f(X_{k}))<1$, and for each $f(X_{k})-f(\widehat{X})>0$,
\begin{align*}
f(X_{k})-f(\widehat{X})&\leq c^{\frac{1}{\theta}}{\rm dist}(\textbf{0}, \partial f(X_{k}))^{\frac{1}{\theta}}\\
&\leq c^{\frac{1}{\theta}}{\rm dist}(\textbf{0}, \partial f(X_{k}))^{\frac{1}{\alpha}}\\
&\leq \frac{c^{\frac{1}{\theta}}}{c_{2}}\left(f(X_{k-j})-f(X_{k})\right).
\end{align*}
Then,
\begin{equation*}
f(X_{k})-f(\widehat{X})\leq\frac{c^{\frac{1}{\theta}}}{c_{2}+c^{\frac{1}{\theta}}}(f(X_{k-j})-f(\widehat{X})).
\end{equation*}
Hence, we have
\begin{equation*}
f(X_{k})-f(\widehat{X})\leq C_{1}\left(\frac{c^{\frac{1}{\theta}}}{c_{2}+c^{\frac{1}{\theta}}}\right)^{\left\lfloor\frac{k-k_{1}}{j}\right\rfloor+1}
\end{equation*}
for each $k\geq k_{1}$, where $C_{1}$ is the same as that in the proof of Theorem \ref{thm:1}. So a $\mathcal{O}(\eta^{k})$ local convergence rate is achieved with $\eta=(c^{\frac{1}{\theta}}/(c_{2}+c^{\frac{1}{\theta}}))^{\frac{1}{j}}$.
With the similar arguments as in the proof of Theorem \ref{thm:1}, we obtain the rest of Theorem \ref{thm:2}.
\end{proof}
It is worth noting that although the local R-linear convergence can be achieved in any cases under A2, when $\theta\in(\frac{1}{2}, 1)$, verification whether a training model and algorithm satisfies the stronger decrease condition is a challenging problem \cite{Li2018,Luo1996,Yu2021}.
\section{Conclusions}
\label{sec:5}
In this paper, a unified framework is proposed to analyze the convergence rate of the objective function value sequences generated by the AM-type training algorithms. The non-monotone $j$-step sufficient decrease conditions and the KL property play central roles in our analysis. And the requirement of nonincreasing property of function value sequence is relaxed in our framework. Based on the squared norm lower bound estimation of the $j$-step descent, three kinds of convergence rates are discussed for different values of the KL exponent $\theta$, respectively. Moreover, if a larger descent is guaranteed, we can improve the convergence rate to $\mathcal{O}(\eta^{k})$ for $\theta\in(\frac{1}{2}, 1)$.

\section*{Acknowledgements}
Bao's research was supported by the National Natural Science Foundation of China (Grant No. 11901338) and the Tsinghua University Initiative Scientific Research Program. Xing's research was supported by the National Natural Science Foundation of China (Grant No. 11771243). The authors would like to thank the editor and anonymous reviewers for carefully reading the manuscript and insightful suggestions.

\bibliographystyle{plain}
\bibliography{ref}
\end{document}